\documentclass[10pt,twocolumn,letterpaper]{article}

\usepackage{cvpr}
\usepackage{times}
\usepackage{epsfig}
\usepackage{graphicx}
\usepackage{amsmath}
\usepackage{amssymb}
\usepackage{float}
\usepackage{amsthm}

\usepackage[pagebackref=true,breaklinks=true,letterpaper=true,colorlinks,bookmarks=false]{hyperref}

\cvprfinalcopy % *** Uncomment this line for the final submission

\def\cvprPaperID{3922} % *** Enter the CVPR Paper ID here

\ifcvprfinal\fi
\begin{document}

%%%%%%%%% TITLE
\title{A Wasserstein GAN model with the total variational regularization}

\author{Lijun Zhang\\
%Shanghai University of Science Engineering\\
Shanghai, China\\
%{\tt\small zhanglj@sues.edu.cn}
\and
Yujin Zhang\\
%Shanghai University of Science Engineering\\
Shanghai, China\\
%{\tt\small yjzhang@sues.edu.cn}
\and
Yongbin Gao\\
%Shanghai University of Science Engineering\\
Shanghai, China\\
%{\tt\small gaoyongbin@sues.edu.cn}
}

\maketitle
%\thispagestyle{empty}

%%%%%%%%% ABSTRACT
\begin{abstract}
It is well known that the generative adversarial nets (GANs) are remarkably difficult to train. The recently proposed Wasserstein GAN (WGAN) creates principled research directions towards addressing these issues. But we found in practice that gradient penalty WGANs (GP-WGANs) still suffer from training instability. In this paper, we combine a Total Variational (TV) regularizing term into the WGAN formulation instead of weight clipping or gradient penalty, which implies that the Lipschitz constraint is enforced on the critic network. Our proposed method is more stable at training than GP-WGANs and works well across varied GAN architectures. We also present a method to control the trade-off between image diversity and visual quality. It does not bring any computation burden.
\end{abstract}

%%%%%%%%% BODY TEXT
\section{Introduction}\label{section:Introduction}
Though there has been an explosion of GANs in the last few years\cite{2016-Radford-DCGAN, 2014-Kingma-SSGAN, 2014-Mirza-CGAN, 2017-zhu-CycleGAN, 2017-Zhang-StackGAN, 2018-Liu}, the training of some of these architectures were unstable or suffered from mode collapse. During training of a GAN, the generator $G$ and the discriminator $D$ keep racing against each other until they reach the Nash equilibrium, more generally, an optimum. In order to overcome the training difficulty, various hacks are applied depending on the nature of the problem. Salimans $et$ $al$. \cite{2016-Salimans} presents several techniques for encouraging convergence, such as feature matching, minibatch discrimination, historical averaging, one-sided label smoothing and virtual batch normalization. In \cite{2017-Arjovsky-WGAN0}, authors explain the unstable behaviour of GANs training in theory and in \cite{2017-Arjovsky-WGAN1} the Wasserstein GANs (WGANs) with weight clipping was proposed. Because the original WGANs still sometimes generate only poor samples or fail to converge, its gradient penalty version (GP-WGANs) was proposed in \cite{2017-Gulrajani-WGAN2} to address these issues. Since then several WGAN-based methods were proposed\cite{2017-Cui, 2018-Adler, 2018-Gemici}. The Boundary equilibrium generative adversarial networks (BEGANs) was proposed in \cite{2017-Berthelot-BEGAN}. Its main idea is to have an auto-encoder as a discriminator, where the loss is derived from the Wasserstein distance between the reconstruction losses of real and generated images:

WGAN requires that the discriminator (or called as critic) function must satisfy 1-Lipschitz condition, which is enforced through gradient penalty in \cite{2017-Gulrajani-WGAN2}. The GP-WGANs are much more stable in training than the original weight clipping version. However, we found in practice that GP-WGANs still have such drawbacks: 1) Fail to converge with homogeneous network architectures. 2) Weights explode for a high learning rate. 3) Losses fluctuate drastically even after long term training. BEGANs improve the training stability significantly and are able to generate high quality images, but the auto-encoder is high resource consuming.

In this paper, we conduct some experiments to show the above defects of GP-WGANs. Then with exploration about the theoretical property of the Wasserstein distance, we choose to add a TV regularization term rather than the gradient penalty term into the formulation of the objective. Compared with GP-WGANs and BEGANs, our approach is simple and effective, but it is much stable for training of varied GAN architectures. Additionally, we introduce a margin factor which is able to control the trade-off between image diversity and visual quality.

We make the following contributions:
1) An Wasserstein GAN model with TV regularization (TV-WGAN) is proposed. It is much simpler but is more stable than GP-WGANs. 2) The TV term implies that the 1-Lipschitz constraint is enforced on the discriminative function. We try to give a rough proof for it. 3) A margin factor is introduced to control the trade-off between generative diversity and visual quality.

The remainder of this paper is organized as follows. In Section \ref{section:Background}, we first introduce the preliminary theory about training of GANs. In Section \ref{section:GP-WGAN defects}, we demonstrate the defects of GP-WGANs by experiments. In Section \ref{section:my method}, we describe the proposed TV regularized method in detail. The implementation and results of our approach are given in Section \ref{section:experiment}. It is summarized in Section \ref{section:conclusion}.

%------------------------------------------------------------------------
\section{Background}\label{section:Background}

\subsection{The Training of GAN}

In training of an original GAN, the generator $G$ and the discriminator $D$ play the following two-player minimax game\cite{2014-Goodfellow-GAN}:
\begin{equation}\label{e1}
 \min_G\max_D \{\mathbb{E}\log D(x)+\mathbb{E}\log (1-D(G(z)))\}
\end{equation}
This objective function of cross entropy is equivalent to the following Jensen-Shannon divergence form:
\begin{equation}\label{e2}
  \min_G\max_D JSD(P_r\|P_g)
\end{equation}
where the real data $x\sim P_r$, the generated data $G(z)\sim P_g$, and $JSD(\cdot)$ measures the similarity of the distribution $P_r$ and $P_g$. When $P_r=P_g$, their JSD is zero, means that the generator perfectly replicating the real data generating process. However, it is always remarkably difficult to train GANs, which suffer from unstable and mode collapse.

\subsection{Wasserstein distance}

As discussed in \cite{2017-Arjovsky-WGAN0}, the distribution $P_r$ and $P_g$ have disjoint supports because their supports lie on low dimensional manifolds of a high dimensional space. Therefore their JSD is always a constant $\log2$, which means that the gradient will be zero almost everywhere. Besides gradient vanishing, GANs can also incur gradient instability and mode collapse problems, if the generator takes $-\log D(G(z))$ as its loss function. In order to get better theoretical properties than the original, WGANs leverages the Wasserstein distance between $P_r$ and $P_g$ to produce an objective function\cite{2017-Arjovsky-WGAN1}:
\begin{equation}\label{e3}
  W(P_r,P_g)=\inf_{\gamma\in(P_r,P_g)}\mathbb{E}_{(x,\widetilde{x})\sim \gamma}\|x-\widetilde{x}\|
\end{equation}
where $x\sim P_r$ and $\widetilde{x}\sim P_g$ implicitly defined by $\widetilde{x}=G(z)$. Since the infimum is highly intractable, with Kantorovich-Rubinstein duality, it can be reformulated as\cite{2017-Gulrajani-WGAN2}:
\begin{equation}\label{e4}
  \min_G\max_{D\in \mathcal{D}} \mathbb{E}D(x)-\mathbb{E}D(G(z))
\end{equation}
where the $\mathcal{D}$ is the set of 1-Lipschitz functions. To enforce the Lipschitz constraint on the discriminator, \cite{2017-Arjovsky-WGAN1} proposes to clip the weights of the discriminator to lie within a compact space $[-c,c]$. In \cite{2017-Gulrajani-WGAN2}, it is found that weight clipping will result in vanishing or exploding gradients, or weights will be pushed towards the extremes of the clipping range. Hence the gradient penalty method is proposed:
\begin{equation}\label{e5}
  L = \mathbb{E}D(x)-\mathbb{E}D(\widetilde{x})+\lambda\mathbb{E}[(\|\nabla_{\widehat{x}}D(\widehat{x})\|_{2}-1)^{2}]
\end{equation}

%-------------------------------------------------------------------------
\section{Difficulties with gradient penalty}\label{section:GP-WGAN defects}

Though the GP-WGAN is theoretically elegant, usually the gradient operation is sensitive to noise. Probably due to this reason, training of GP-WGANs is still problematic. We illustrate this by running experiments on image generation using the GP-WGAN algorithm.

\subsection{Unstable for the homogeneous network architecture}\label{section:sub-homogeneous-network}

The training of GP-WGAN will be unstable for such a network architecture as shown in Figure~\ref{fig:network-1}. It is similar to that of the BEGAN\cite{2017-Berthelot-BEGAN}, except that our discriminator does not have an auto-decoder inside it. This structure is homogeneous since it is composed of repeated convolutional layers. Usually, GANs with such homogeneous structure is hard to train due to its lack of either batch normalization\cite{2015-Ioffe-BatchNorm} or dropout layers. Therefore this structure can be taken as a benchmark for evaluating the training stability of a GAN model. We trained the GP-WGAN with this network structure. As we observed in our experiments, it resulted in exploding gradients in GP-WGAN as shown in Figure~\ref{fig:gradient-explode}.

\begin{figure}[t]
\begin{center}
   \includegraphics[width=1.0\linewidth]{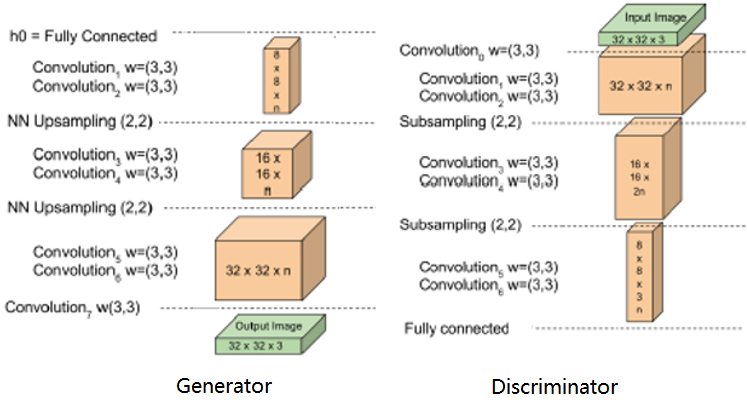}
\end{center}
   \caption{A homogeneous network structure which is similar to BEGANs.}
\label{fig:network-1}
\end{figure}

\begin{figure}[t]
\begin{center}
   \includegraphics[width=1.0\linewidth]{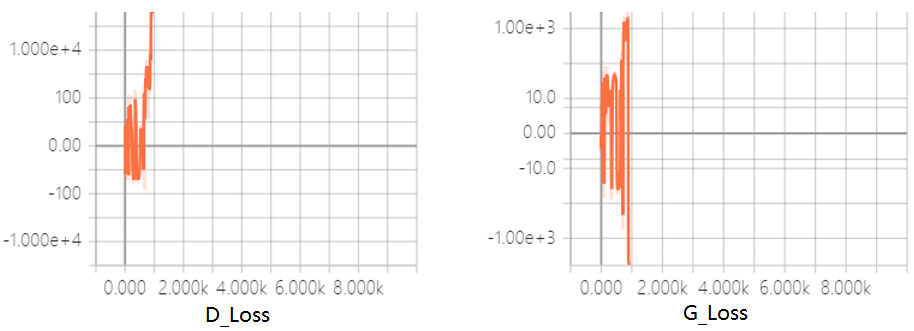}
\end{center}
   \caption{Gradient explosion of GP-WGANs. This happens when 1) using a homogeneous network structure, or 2) the learning rate is high}
\label{fig:gradient-explode}
\end{figure}

\subsection{Drastic fluctuation in long term}\label{session:sub-dcgan-network}

We tested the GP-WGAN with a carefully designed structure as shown in figure \ref{fig:network-2}. This network is similar to that of the standard DCGAN\cite{2016-Radford-DCGAN}, where the batch normalization is applied in both the generator and the discriminator. It is much easy to train this network because its structure is carefully designed for GANs. Not surprisingly, the GP-WGAN with this structure is capable of being trained to generate good samples. However, the loss curve of the GP-WGAN keeps intense fluctuating even after a long term training, which indicates that the training of the GP-WGAN is potentially unstable. In our experiments, this phenomenon has recurred repeatedly for both CIFAR-10 and CelebA datasets, as shown in Figure \ref{fig:loss-fluctuate}.

\begin{figure}[t]
\begin{center}
   \includegraphics[width=1.0\linewidth]{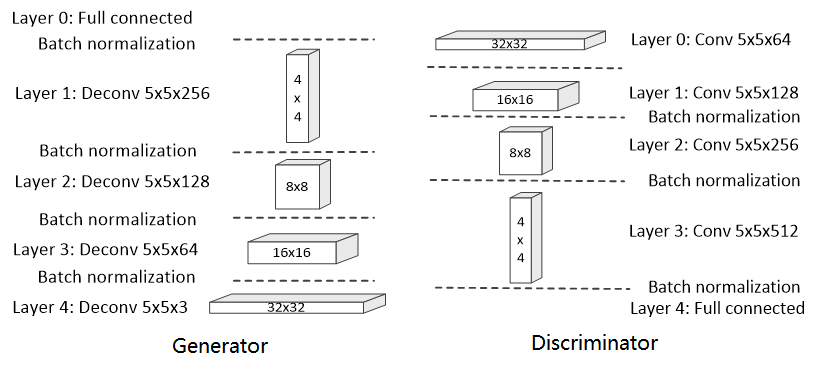}
\end{center}
   \caption{A carefully designed network structure similar to DCGANs.}
\label{fig:network-2}
\end{figure}

\begin{figure}[t]
\begin{center}
   \includegraphics[width=1.0\linewidth]{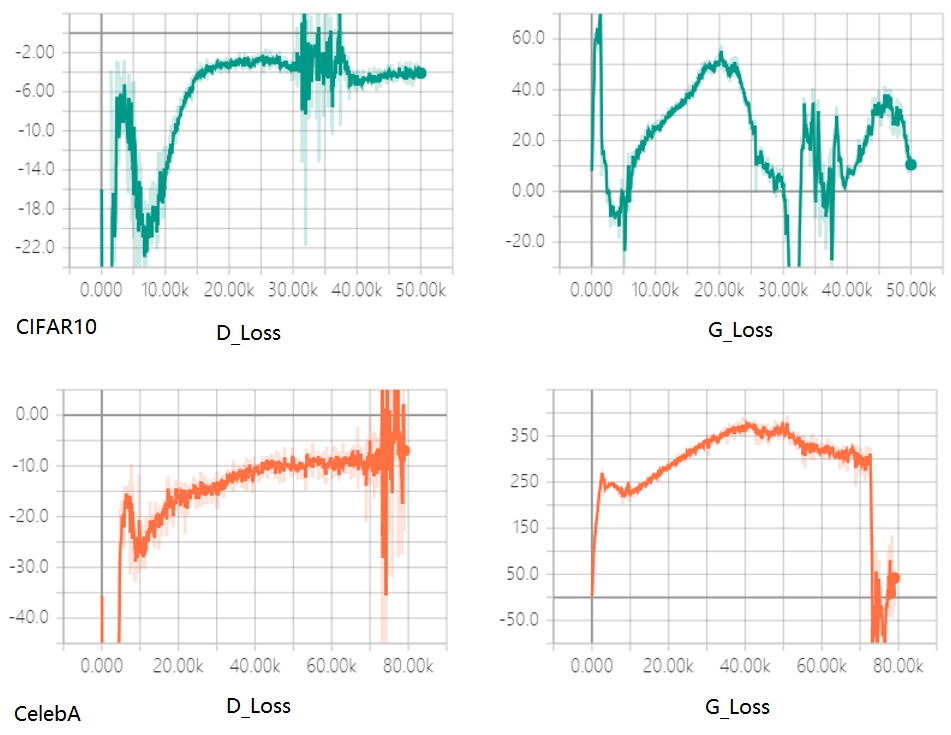}
\end{center}
   \caption{Drastic fluctuation of GP-WGANs after a long term traing with learning rate=1e-5.}
\label{fig:loss-fluctuate}
\end{figure}

\subsection{Sensitive to learning rate}\label{session:sub-high-learning-rate}

The third drawback of GP-WGAN is sensitivity to the learning rate of training. In the last experiment, if we enlarge the learning rate from 1e-5 to 1e-4 which is suitable for most GAN models including BEGANs, the GP-WGAN becomes unstable and gradients explode rapidly. The loss curves are the same as in Figure~\ref{fig:gradient-explode}, so we don't re-draw it here. Certainly the learning rate plays a significant role in most deep learning networks. A too high learning rate always makes them difficult to converge, but rarely leads gradient explosion. This illustrates that GP-WGANs are potentially unstable.

%------------------------------------------------------------------------
\section{Proposed method}\label{section:my method}

In this section, we first present the TV-WGANs along with their benefits in section \ref{section:sub-TV-WGAN}. A rough proof for the enforcement of Lipschitz constraint by the TV term is given in section \ref{section:sub-lipschitz}. Then in section \ref{section:sub-my-objective} we present the objective function with a margin factor which controls the trade-off between generative diversity and visual quality. Finally, the benefits of our approach are discussed in section \ref{section:sub-my-network}.

\subsection{Total variational WGAN}\label{section:sub-TV-WGAN}

The WGAN objective Equation (\ref{e4}) could be explained in a intuitive way: The discriminator $D$ is trained to make its output as high as possible for real data $x$, and as low as possible for fake data $\widetilde{x}$; The generator $G$ which are trained to produce fake images tries to make the discriminator to give a high output for them. Since the discriminator works as a real-fake critic in a WGAN model, it tries to make its output separated as far as possible for real data and fake data. As a counterpart in the adversarial game, the generator tries to make the discriminative output as close as possible to each other. Besides, Equation (\ref{e4}) implies that the discriminator must be a smooth 1-Lipschitz function. Without this constraint, the training will diverge continuously. As shown in Figure \ref{fig:no-Lipschitz}, the $D(x)$ and $D(\widetilde{x})$ tends to increase and decrease infinitely until exceed the range of machines floating point number.

\begin{figure}[t]
\begin{center}
   \includegraphics[width=1.0\linewidth]{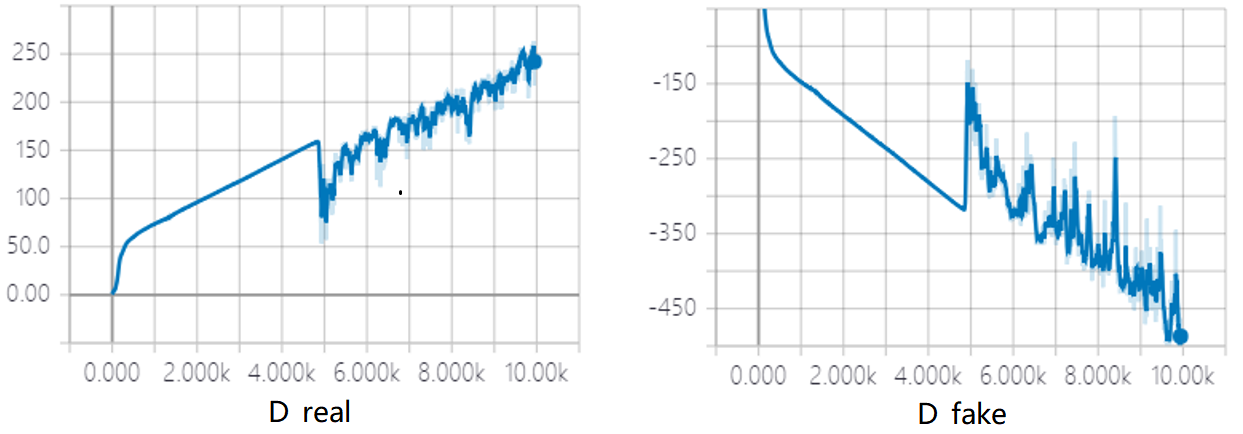}
\end{center}
   \caption{Discriminator output in training for real data and fake data without 1-Lipschitz constraint.}
\label{fig:no-Lipschitz}
\end{figure}

To enforce 1-Lipschitz constraint on discriminative functions, the weighting clipping and the gradient penalty methods are proposed. However, both these methods have their defects as demonstrated in section~\ref{section:GP-WGAN defects}. In this section we introduced an alternative in the form of total variational regularization $\mid D(x)-D(\widetilde{x})-\delta\mid$, where $\delta$ is a wanted margin between discriminative output for real data and fake data. The meaning of the margin $\delta$ is shown in Figure~\ref{fig:magin-factor}. Hence the objective function is formulated as:
\begin{equation}\label{e6}
  \min_G\max_D \mathbb{E}D(x)-\mathbb{E}D(\widetilde{x})+\lambda\mathbb{E}[\mid D(x)-D(\widetilde{x})-\delta\mid]
\end{equation}
where $\lambda$ is the regularization factor which is set to 1 across this paper. We assume that it is equivalent to enforce an 1-Lipschitz constraint via this TV term. About this point, a simple proof will be given in the next section.

\begin{figure}[t]
\begin{center}
   \includegraphics[width=0.6\linewidth]{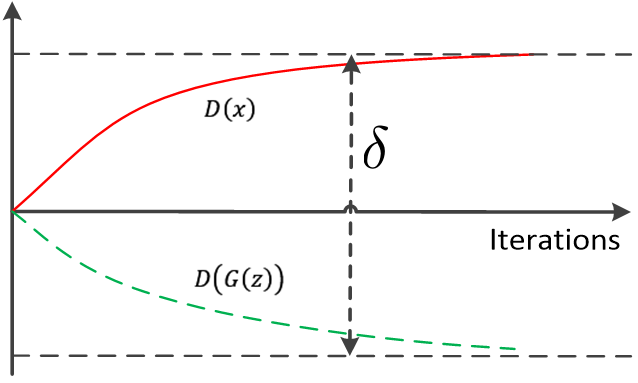}
\end{center}
   \caption{The margin between discriminative outputs of real data and fake data.}
\label{fig:magin-factor}
\end{figure}

\subsection{The 1-Lipschitz constraint}\label{section:sub-lipschitz}
\newtheorem{mypro}{Proposition}
\begin{mypro}
The marginal TV regularizing term $\mid D(x)-D(\widetilde{x})-\delta\mid$ in equation~\ref{e6} enforces the 1-Lipschitz constraint on the discriminative function $D(x)$.
\end{mypro}
\begin{proof}
Given a discriminative function $D(\theta)$, whose parameters $\theta$ are updated by the gradient descent method:
\begin{equation}\label{e7}
 \theta_{n+1}=\theta_n+\eta\nabla\theta_n
\end{equation}
where $\eta$ is the learning rate. Then by the first order Taylor expansion, the discriminative output $D_{n+1}(x)=D(\theta_n+\eta\nabla\theta_n,x)$ for real data can be approximated as:
\begin{equation}\label{e8}
 \begin{aligned}
 D_{n+1}(x)\approx D(\theta_n,x)+\eta\nabla\theta_nD'(\theta_n) \\
 =D_n(x)+\eta\nabla_xD(x)
 \end{aligned}
\end{equation}
On the other hand, its output for fake data is approximately:
\begin{equation}\label{e9}
 D_{n+1}(\widetilde{x})=D_n(\widetilde{x})+\eta\nabla_{\widetilde{x}}D(\widetilde{x})
\end{equation}
Considering that the discriminator is trained to give a continuously increasing output for real data, and to give a continuously decreasing output for fake data, hence
\begin{equation}\label{e10}
 \begin{cases}
 \mathbb{E}[\nabla_xD(x)]>0 \\
 \mathbb{E}[\nabla_{\widetilde{x}}D(\widetilde{x})]<0 \\
 \end{cases}
\end{equation}
Meanwhile, the difference between $D_n(x)$ and $D_n(\widetilde{x})$ is bounded by the marginal TV regularization, $\ie$, for a constant $\epsilon$, we have
\begin{equation}\label{e11}
  \mid D_n(x)-D_n(\widetilde{x})-\delta\mid <\epsilon
\end{equation}
So by subtracting Equation~(\ref{e8}) from (\ref{e9}), we derive that the difference between $\nabla_xD(x)$ and $\nabla_{\widetilde{x}}D(\widetilde{x})$ is also bounded by
\begin{equation}\label{e12}
  \mid\nabla_xD(x)-\nabla_{\widetilde{x}}D(\widetilde{x})\mid < \frac{2\epsilon}{\eta}
\end{equation}
Considering Equation~(\ref{e10}), we derive that $\nabla_xD(x)$ must also be bounded by $\epsilon/\eta$, which means that $D(x)$ must satisfy the $k$-Lipschitz constraint for $k=\epsilon/\eta$. If $\epsilon$ is small enough, which is controlled by the regularizing factor $\lambda$, then the 1-Lipschitz constraint is enforced.
\end{proof}

In this section, we prove that the 1-Lipschitz constraint is enforced approximately by the TV regularization. Additionally, we draw the histogram of the discriminator¡¯s weights in Figure~\ref{fig:weight-hist} after we train TV-WGAN using CIFAR-10 dataset. As shown in the figure, the weights of the TV-WGAN remain uniform, unlike in the weight clipping WGAN model where weights are pushed towards two extremes of the clipping range.

\begin{figure}[t]
\begin{center}
   \includegraphics[width=0.6\linewidth]{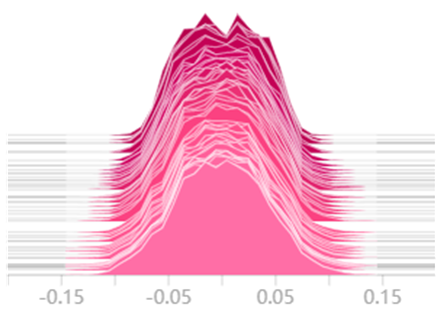}
\end{center}
   \caption{Weights histogram of discriminator in TV-WGAN.}
\label{fig:weight-hist}
\end{figure}

\subsection{Objective and margin factor}\label{section:sub-my-objective}
In Equation~\ref{e6}, the purpose of the TV term is to make the discriminative output of real data and fake data to be separated from each other within a bound. All experiments in this paper use $\lambda=1$. The larger $\lambda$, the stronger the Lipschitz constraint is enforced. We rewrite the loss function of our model as:
\begin{equation}\label{e13}
 \begin{cases}
 L_D=-\mathbb{E}D(x)+\mathbb{E}D(G(z))+\lambda\mathbb{E}|D(x)-D(G(z))-\delta| \\
 L_G=-\mathbb{E}D(G(z)) \\
 \end{cases}
\end{equation}
The margin factor $\delta$ is capable of controlling the trade-off between generative diversity and visual quality. Higher values of $\delta$ lead to higher visual quality because it helps distinguish real data and fake data, so that the generator has to output vivid images with more details. Lower values of $\delta$ lead to higher image diversity.

\subsection{Benefits}\label{section:sub-my-network}
The benefits of TV-WGANs can be derived from two aspects. First, unlike GP-WGANs which directly compute gradients as penalty, the TV-WGANs enforce the Lipschitz constraint via TV regularization which is much more mild than weight clipping and gradient penalty. This results in more stable gradients that neither vanish nor explode, allowing training of more complicated networks, as well as more homogeneous networks.

Second, TV-WGANs do not adding any computation burden. On the contrary, GP-WGANs have to implement the troublesome gradients operation. BEGANs do not solve gradients for Lipschitz constraint, but their discriminator is composed of an auto-encoder and an auto-decoder, which means that they require even more computation resources.

\section{Experiments}\label{section:experiment}
In this section, we run experiments on image generation using our TV-WGAN algorithm and show that there are significant practical benefits to using it over other Wasserstein GANs.

\subsection{Set up}
We conduct experiments on the CIFAR-10 and CelebA datasets whose resolutions are $32\times32$ and $128\times96$ respectively. As for the CelebA dataset, we scale each picture into $64\times48$ for training. We trained our model using the Adam optimizer with an constant learning rate of 1e-4 for CIFAR-10 dataset and 1e-5 for CelebA dataset. Mode collapse will be observed if a high learning rate is used. However gradient explosion never happens in training of our model, unlike the GP-WGAN which is certain to suffer from gradient explosion if using high learning rate.

\subsection{Stability}
We use CIFAR-10 dataset to train the TV-WGAN model with the DCGAN-like network structure which is depicted in section~\ref{session:sub-dcgan-network} and is shown in Figure~\ref{fig:network-2}. The CelebA dataset is used to train the BEGAN-like network which is depicted in section~\ref{section:sub-homogeneous-network} and is shown in Figure~\ref{fig:network-1}. The resulted loss curves are draw in Figure~\ref{fig:stable-dcgan} and \ref{fig:stable-began}. Compared with Figure~\ref{fig:loss-fluctuate}, we believe that our model is much more stable in training than the GP-WGAN which keeps drastic fluctuating even after very long term training.

Further experiments demonstrate that training the GP-WGAN always encounters gradient explosion when using BEGAN-like network structures or using high learning rate, as we have discussed in section~\ref{section:sub-homogeneous-network} and \ref{session:sub-high-learning-rate}. This indicates that the gradient penalty did not make the WGANs stable enough, probably because the derivative control is susceptible to noise. While our TV-WGAN model has never been observed for gradient explosion in all above situations.

We have not so far done more experiments with many other network architectures, but the two structures we used here are typical. The DCGAN-like structure is carefully designed for GANs, with batch normalization layers in both generator and discriminator, so it is actually easy to train. The BEGAN-like structure is too homogeneous and is hard to train due to lack of either batch normalization or dropout layers.

\begin{figure}[t]
\begin{center}
   \includegraphics[width=1.0\linewidth]{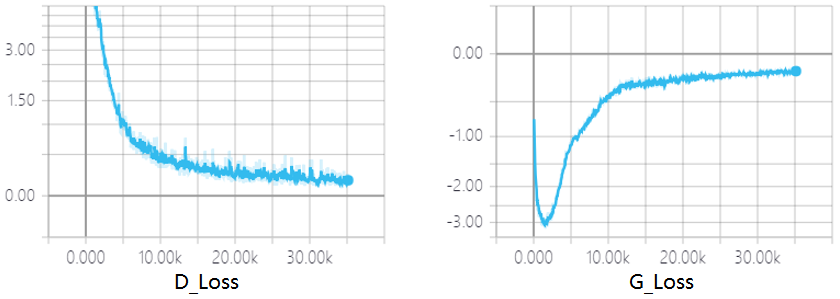}
\end{center}
   \caption{Training TV-WGAN on CIFAR-10 dataset with learning rate=1e-4.}
\label{fig:stable-dcgan}
\end{figure}

\begin{figure}[t]
\begin{center}
   \includegraphics[width=1.0\linewidth]{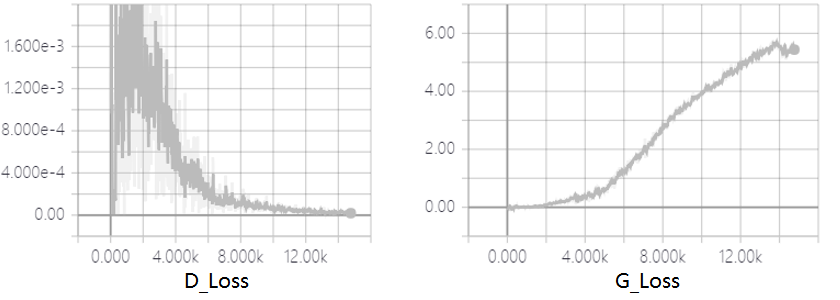}
\end{center}
   \caption{Training TV-WGAN on CelebA dataset with learning rate=1e-5.}
\label{fig:stable-began}
\end{figure}

Figure~\ref{fig:celeba-quality} shows the generated face pictures by the TV-WGAN with the homogeneous network. The generated images do not seem to be of high quality, because we only trained 20 epochs.

\begin{figure}[t]
\begin{center}
   \includegraphics[width=1.0\linewidth]{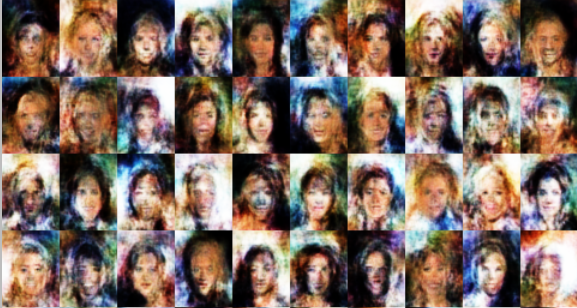}
\end{center}
   \caption{The generated face pictures by the TV-WGAN with 20 epochs.}
\label{fig:celeba-quality}
\end{figure}

\subsection{Effect of margin factor}
In section~\ref{section:sub-my-objective}, we introduced a margin factor $\delta$. Figure~\ref{fig:effect-margin} demonstrate its effect on the generative diversity and visual quality. The effect is obvious when training 50 epochs, that higher values of $\delta$ lead to higher visual quality.

When our model is trained for 100 epochs, the quality of the generated images is not very different from each other for various margin factors. It's hard to distinguish by naked eyes. In section~\ref{section:sub-incept-score}, we will perform the numerical assessment by using inception scores. As shown in table~\ref{tab:incept-score}, the quality of generated images improves with the increase of the margin factor.

\begin{figure}[t]
\begin{center}
   \includegraphics[width=1.0\linewidth]{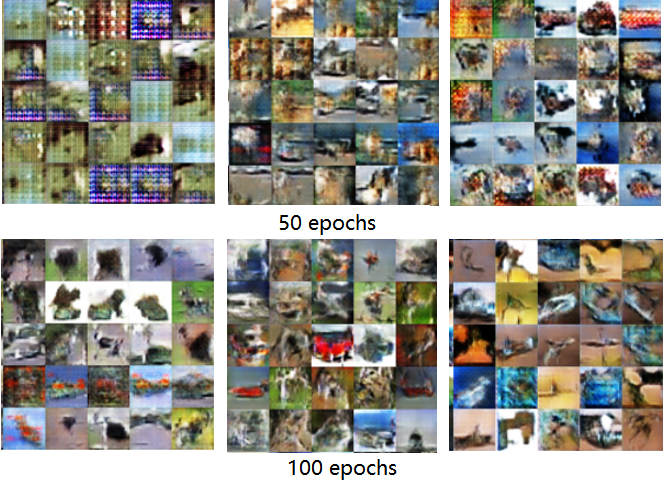}
\end{center}
   \caption{The effect of the margin factor. From left to right, $\delta=0,5,10$}
\label{fig:effect-margin}
\end{figure}

\subsection{Comparison with the BEGAN model}\label{section:sub-compare-began}
The BEGAN is capable of producing high quality images. However, the BEGAN model is prone to mode collapse when it is over-trained, which is verified by our experiments. We trained a TV-WGAN model and a BEGAN model on the CIFAR-10 dataset respectively by using the learning rate of 1e-4. When they are trained for 20 epochs, generated images of our model are of low quality, while the BEGAN model can generate good images, though they are lack of diversity. When we train them for 50 epochs, both models produce good images and the BEGAN generated pictures are of higher quality. But we found that when they are trained for 100 epochs, our model can generate high quality pictures while the BEGAN falls into the mode collapse if it does not use the learning rate decaying. The results are shown in Figure~\ref{fig:compare-began}.

\begin{figure}[t]
\begin{center}
   \includegraphics[width=1.0\linewidth]{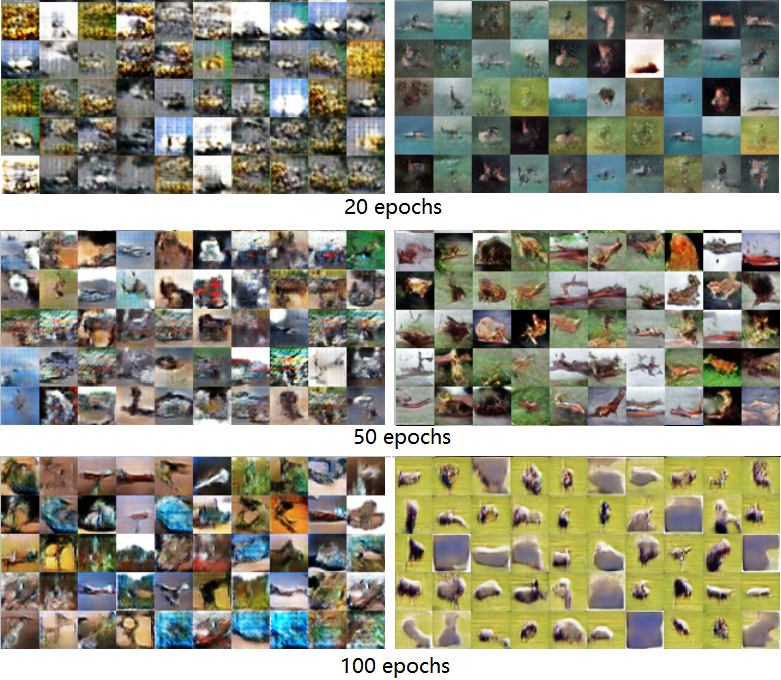}
\end{center}
   \caption{Generated images by the TV-WGAN (left) and the BEGAN model without learning rate decaying (right).}
\label{fig:compare-began}
\end{figure}

\subsection{Inception scores}\label{section:sub-incept-score}
The inception score was considered as a good assessment for sample quality from a dataset\cite{2016-Salimans}. To measure quality and diversity numerically, we trained our model on CIFAR-10 with various margin factor and computed the inception score of generated images. We also trained other models, such as DCGAN, GP-WGAN and BEGAN respectively, and computed their inception scores by ourselves. All of these models are trained by Adam optimizers. The learning rate is 1e-4 for all models except that the GP-WGAN uses 1e-5, since it is unstable with high learning rates.

The results are listed in table \ref{tab:incept-score}. It is shown that the inception scores of our model are higher than that of GP-WGANs, but is still less than DCGANs and BEGANs. The BEGAN model obtains the highest inception score in our experiments, which indicates that it is capable of producing high quality images.

For different $\delta$ value, the TV-WGAN has different inception scores. With the increase of the $\delta$ value, the inception score of the TV-WGAN has also increased, while the standard variance of the score reduced. This seems indicating that the margin factor $\delta$ takes effect for controlling the visual quality of generated images.

\begin{table}
\begin{center}
\begin{tabular}{ccc}
\hline
Models& mean& std\\
\hline
GP-WGAN& 3.75& 0.32\\
DCGAN& 4.53& 0.26\\
BEGAN& 4.83& 0.36\\
TV-WGAN($\delta=0$)& 4.12& 0.42\\
TV-WGAN($\delta=5$)& 4.29& 0.36\\
TV-WGAN($\delta=10$)& 4.47& 0.23\\
\hline
%\caption{inception scores}
\end{tabular}
\end{center}
\caption{Inception scores.}
\label{tab:incept-score}
\end{table}

\section{Conclusion}\label{section:conclusion}

In this work, we demonstrated problems with gradient penalty in WGAN and introduced an alternative in the form of a total variational regularization in the objective function, which enforce the 1-Lipschitz constraint on the discriminator implicitly. The new approach is much stable in training. Additionally, we introduced a margin factor to control the trade-off between generative diversity and visual quality.

{\small
\bibliographystyle{unsrt}
\bibliography{reference}
}
\end{document}